\documentclass{article}

\RequirePackage{amsthm,amsmath,amssymb}
\PassOptionsToPackage{sort}{natbib}

\usepackage{shortcuts}

\usepackage{microtype}
\usepackage{graphicx}
\usepackage{subfigure}
\usepackage{booktabs} %
\usepackage{enumitem}

\usepackage{hyperref}
\usepackage[capitalise]{cleveref}

\usepackage[accepted]{icml2019}

\theoremstyle{plain}
\newtheorem{theorem}{Theorem}

\newtheorem{lemma}[theorem]{Lemma}

\theoremstyle{definition}
\newtheorem{assumption}{Assumption}

\allowdisplaybreaks

\begin{document}

\twocolumn[
\icmltitle{Classifying Treatment Responders Under Causal Effect Monotonicity}

\icmlsetsymbol{equal}{*}

\begin{icmlauthorlist}
\icmlauthor{Nathan Kallus}{cornell}
\end{icmlauthorlist}

\icmlaffiliation{cornell}{School of Operations Research and Information Engineering and Cornell Tech, Cornell University}

\icmlcorrespondingauthor{Nathan Kallus}{kallus@cornell.edu}

\icmlkeywords{Causal inference, Individual-level effect estimation, Conditional average treatment effects, classification, surrogate loss}

\vskip 0.3in
]

\printAffiliationsAndNotice{}  %

\begin{abstract}
In the context of individual-level causal inference, we study the problem of predicting whether someone will respond or not to a treatment based on their features and past examples of features, treatment indicator (e.g., drug/no drug), and a binary outcome (e.g., recovery from disease). As a classification task, the problem is made difficult by not knowing the example outcomes under the opposite treatment indicators. We assume the effect is monotonic, as in advertising's effect on a purchase or bail-setting's effect on reappearance in court: either it would have happened regardless of treatment, not happened regardless, or happened only depending on exposure to treatment. Predicting whether the latter is latently the case is our focus. While previous work focuses on conditional average treatment effect estimation, formulating the problem as a classification task rather than an estimation task allows us to develop new tools more suited to this problem. By leveraging monotonicity, we develop new discriminative and generative algorithms for the responder-classification problem. We explore and discuss connections to corrupted data and policy learning. We provide an empirical study with both synthetic and real datasets to compare these specialized algorithms to standard benchmarks.
\end{abstract}

\section{Introduction}\label{sec:intro}

In many domains where personalization is of interest, such as healthcare and marketing, a central problem is individual-level causal inference on treatment effects, which are the differences in outcome if a treatment is applied and if not applied.
The aim is to learn a function that, given a rich set of features describing an individual, predicts the causal effect of an intervention on the individual, such as the effect of a pharmaceutical drug on their mortality or of an advertisement on whether they purchase the product.
Compared to aggregate average causal effects on a whole population,
such fine-grained predictions can better describe how an intervention would affect a specific individual and help determine whether it should be applied in their case. %
Learning such a function
from either experimental or observational data
has been the subject of much recent research.
\citep[among others; see \cref{sec:related}]{kunzel2017meta,shalit2017estimating,wager2017estimation}.
The key difficulty in this task arises due to what is often termed the \emph{Fundamental Problem of Causal Inference}: that for any individual in the data, the data only contains the outcome given that the treatment was in fact applied or not, and it does not contain the counterfactual outcome under the opposite scenario. This difficulty arises in both experimental and observational data, although the former has the benefit of a priori eliminating any potential additional biases due to treatment selection via randomization.

In many applications, the outcomes of interest are binary. In medicine, we are often interested in mortality, recovery, readmission, or disease remission. In advertising, we are interested in whether or not a user purchases, visits, clicks, etc. The same holds in many other applications in domains ranging from criminal justice to education policymaking. In deciding whether to release a defendant on their own recognizance (and not require bail), a judge is interested in whether or not a given defendant will fail to reappear in court. 
In designing job training programs, a key outcome of interest is whether the recipient secures employment afterward or not. 
Moreover, in many applications and for many outcomes of interest, the treatment's effect can only be monotonic: 
while it may or may not have any effect,
exposing someone to an advertisement does not make them \emph{less} likely to visit;
while anticoagulants may have an unknown effect on an individual stroke patient's mortality, it can only make the occurrence of a haemorrhage \emph{more} likely;
and
requiring to post bail does not make a defendant \emph{less} likely to reappear.

When outcomes are binary and effects are monotonic, 
the individual-level causal inference question boils down to
just whether a given individual will respond to the treatment or not.
In all of the above settings, either the outcome event of interest would have happened anyway, not happened anyway, or happened if and only if treatment was applied.
We call instances that fall in the latter category responders and those in the former two category non-responders. 
An instance can be finer than an individual because it refers to a particular realization, where an individual could have a probability of being in any one of these categories.

In this paper, we study the problem of learning to classify responders in settings with binary outcomes and monotonic effects. This is a unique classification problem as it suffers from the problem that the labels are not observed in the training data: due to the fundamental problem of causal inference, we do not know the counterfactual outcome bit and whether it is the same as the observed outcome or flipped. 
By leveraging monotonicity we develop a new discriminative approach based on minimizing a surrogate loss for the responder-classification task. Using a hinge loss and kernelizing the decision function, this gives rise to an algorithm we term RespSVM. We discuss the approach from a corrupted-label perspective as well as what happens if monotonicity fails. Based on the corrupted-label perspective, we further develop a new generative approach that gives rise to a new cross-entropy loss that we use in an algorithm we term RespNet. 
We then explore how these algorithms compare to standard benchmarks from individual-level effect estimation. Our empirical study includes both synthetic and real-data examples and shows that, when outcomes are binary and classifying response is of interest, specialized algorithms such the ones we develop can provide better performance.

\subsection{Problem setup}

We consider 
a population of instances where each instance is associated with the following random variables:
\begin{itemize}[leftmargin=*,,labelindent=0in,topsep=-1ex,itemsep=0ex,partopsep=0ex,parsep=0ex]
\item $X\in\R p$, features to be used to predict outcomes and treatment response (also known as baseline covariates);
\item $Y(+1)\in\{-1,+1\}$, outcome if 
treatment is applied;
\item $Y(-1)\in\{-1,+1\}$, outcome if 
treatment is \emph{not} applied.
\end{itemize}
Note that we can also conceive of $Y(+1),Y(-1)$ as the potential outcomes of any two alternative interventions, $+1$ and $-1$. Here we identify intervention $+1$ with applying the treatment and $-1$ with \emph{not} applying treatment only for the sake of exposition.
Note also that if we would rather consider an instance as having some \emph{probability} of having some particular outcomes rather than having certain binary outcomes, we can just simply augment the population appropriately with each binary-outcome scenario. 

The causal effect of treatment in each instance is defined as 
the difference in outcome if treatment is applied or not
$$\text{Causal effect: $Y(+1)-Y(-1)$}$$
Our standing assumption, as motivated in the introduction, is that the 
causal effect is nonnegative:
\begin{assumption}[Monotonicity]\label{asm:monotonicity}
$Y(+1)\geq Y(-1)$.
\end{assumption}
Note that if our assumed monotonicity went the other way (treatment can only decrease outcome or keep it the same), we can simply negate the outcome (\ie, swap the physical meanings of having a $+1$ or $-1$ outcome) in order to conform to \cref{asm:monotonicity}.

Under \cref{asm:monotonicity}, we can exhaustively classify units into three categories:
\begin{itemize}[leftmargin=*,,labelindent=0in,topsep=-1ex,itemsep=0ex,partopsep=0ex,parsep=0ex]
\item Responders: $Y(+1)=+1,\,Y(-1)=-1$
\item Type-1 non-responders: $Y(+1)=Y(-1)=-1$
\item Type-2 non-responders: $Y(+1)=Y(-1)=+1$
\end{itemize}
\Cref{asm:monotonicity} simply eliminates the fourth possibility of having $Y(+1)=-1,\,Y(-1)=+1$. Notice that all responders have a causal effect of $2$ and all non-responders have a causal effect of $0$.

\subsection{The data and the classification task}

Letting 
$$
R=\pw{+1&\quad\text{$Y(+1)>Y(-1)$ (responder)}\\-1&\quad\text{Otherwise (non-responder)}}
$$
and
$$
\rho(X)=\Prb{R=+1\mid X},
$$
we consider the classification task of predicting whether a unit is a responder or not. That is, the binary classification task with features $X$ and binary label $R$.

The training data we have for this classification task does \emph{not} consist of example pairs of features and labels, however.
Instead, the training data
consists of $n$ observations of units that were either exposed or not to the treatment and the outcome corresponding to this exposure. Specifically, our observations are of the random variables
\begin{itemize}[leftmargin=*,,labelindent=0in,topsep=-1ex,itemsep=0ex,partopsep=0ex,parsep=0ex]
\item $X\in\R p$, features as before;
\item $T\in\{-1,+1\}$, an indicator of
whether the unit was exposed ($+1$) or not ($-1$) to treatment; and
\item $Y=Y(T)\in\{-1,+1\}$, the corresponding outcome.
\end{itemize}
And, our training data consist of $n$ observations, $X_1,T_1,Y_1,\dots,X_n,T_n,Y_n$, of the variables $X,T,Y$.

We focus on the case where this data came from an experiment.
We therefore assume
that treatment selection is unconfounded in that 
$$Y(+1),Y(-1)\indep T\mid X,$$ as would be the case under randomization.
We further define the randomization probability
\begin{align*}
e(X)=\Prb{T=1\mid X}\quad\text{and}\quad Q=\textstyle\frac12+\prns{e(X)-\frac12}T.
\end{align*}
For experimental data, $e(X)$ is known by design and is often constant, usually equal to $1/2$.
Observational data are characterized by the setting where unconfoundedness is an assumption rather than a design choice 
and $e(X)$ is unknown. 
One reduction for using any of the approaches we discuss on observational data is to assume unconfoundedness holds and estimate $e(X)$ from the data and impute its value. There may also be other reductions, for example leveraging orthogonalized (doubly robust) estimation \citep{chernozhukov2016double}, but for the sake of clarity we focus on known $e(X)$ and $Q$.
Note also that \cref{asm:monotonicity} is necessary for the identifiability of $\rho$ \citep{imbens1994identification,manski1997monotone,tian2000probabilities}.

Given the above data, we are interested in learning a classifier $f:\R p\to\fbraces{-1,+1}$ to predict $R$ from $X$.
To assess the quality of classifiers,
we focus on the (weighted) misclassification loss.
For $\theta\in[0,1]$, define
\begin{align*}L_\theta(f)&=\theta\Prb{\text{false positive}}+(1-\theta)\Prb{\text{false negative}}\\&=\theta\Prb{f(X)=+1,R=-1}\\&\phantom{=}+(1-\theta)\Prb{f(X)=-1,R=+1}.\end{align*}
We will usually focus on the case $\theta=1/2$, for which
$L_{1/2}(f)=(1-\op{Accuracy}(f))/2$.
Note that, given the true conditional probability $\rho(X)$,
the minimizer of $L_\theta(f)$ over all functions $f$, also known as the Bayes-optimal classifier, is $f^*_\theta(X)=\op{sign}\prns{\rho(X)-\theta}$.
This gives the standard reduction of the classification problem to estimating and thresholding conditional probabilities of labels.
However, estimating these conditional probabilities may not be necessary for successful classification and may not be the best approach.

\subsection{Relationship to CATE}\label{sec:cate}

The conditional average treatment effect (CATE)
is 
the conditional expectation of the causal effect given features:
\begin{equation}\label{eqn:CATE}
\tau(X)=\Efb{Y(+1)-Y(-1)\mid X}=\Efb{YT/Q\mid X},
\end{equation}
where the latter equality arises immediately from 
unconfoundedness
\citep{athey2016recursive}.
As a conditional expectation, CATE can be understood as the \emph{best} predictor of the causal effect in terms of squared error over all functions of $X$. 
CATE can of course be defined as in \cref{eqn:CATE} even if outcomes are not binary. As reviewed in the next section, learning CATE from observations of $X,T,Y$ has been the subject of much recent research.

When outcomes are binary and effects monotonic, we have the following relationship:
\begin{lemma}\label{lemma:cate} Under \cref{asm:monotonicity},
$
\rho(X)
=\tau(X)/2$.
\end{lemma}
\begin{proof} Since the causal effect is 2 for responders and 0 for non-responders, we have $Y(+1)-Y(-1)=2\indic{R=+1}$ and hence
$2\rho(X)=
\Efb{2\indic{R=+1}\mid X}=\tau(X)$.
\end{proof}

This allows for a na\"ive reduction from learning $f^*_\theta(X)$ to learning $\tau(X)$ using a plug-in approach: given an estimate $\hat\tau$ of $\tau$, return $\hat f(X)=\op{sign}\prns{\hat\tau(X)-2\theta}$. This, however, does not directly optimize the classification loss and may fail in producing asymptotically optimal classifiers if $\hat\tau$ is not consistent for $\tau$. 
In our empirical results (\cref{sec:empirics}), we will use this reduction to benchmark our algorithms against a variety of existing CATE-learning algorithms.

\subsection{Related literature}\label{sec:related}

Many recent advances have been made for the important problem of estimating CATE from $X,T,Y$ data. One basic approach to estimating CATE is to estimate $\Efb{Y\mid X,T=+1}$ using some regression method on the treated data, similarly estimate $\Efb{Y\mid X,T=-1}$ on the untreated data, and return the difference, which is sometimes known as ``T-learner'' \citep{kunzel2017meta}. More sophisticated methods attempt to learn the difference directly. \citet{athey2016recursive,wager2017estimation} study adapting tree- and forest-based methods to this problem. \citet{johansson2018learning,shalit2017estimating} develop a neural network architecture for learning CATE with a shared representation as well as generalization bounds that motivate new regularizers. \citet{kunzel2017meta,nie2017learning} develop meta-learners that combine base learners for the outcome regressions and treatment model to learn CATE.
\Cref{asm:monotonicity} also implies the shape constraint $\tau(X)\geq0$, which can be used as a constraint to improve CATE estimation \citep{aronow2013beyond,huang2012assessing}.
All of the above methods estimate $\tau$. These estimates can then be used to classify responders based on the above plug-in approach. However, this does not directly address the misclassification loss.

Another strand of literature has focused on the problem of policy evaluation and learning from $X,T,Y$ data \citep{kallus2018balanced,kallus2018confounding,kallus2018policy,swaminathan2015counterfactual,strehl2010learning,dudik2011doubly,bottou2013counterfactual}. In policy evaluation the target is to estimate the average outcome that would be induced in the population if a certain policy were implemented, that is a mapping from covariates to treatment assignment. In policy learning the target is to find a policy with large average outcome. 
In \cref{sec:policylearning}, we explain that our discriminative classifiers essentially arise from formulating the classification problem as a policy learning problem.

Monotonicity is also a common assumption that arises in instrumental variable (IV) analysis with binary instruments and treatments \citep{angrist2008mostly}. 
In such models, we assume that there is an instrument (\eg, encouragement to enroll in a program) that only affects the outcome (\eg, some measurement after program) via its effect on treatment (\eg, enrollment). The instrument 
is often assumed to be monotonic in its effect on treatment take-up 
and instrument responders are known as compliers (here we instead use ``responder'' because we focus on effect on outcomes). If the instrument is valid and its effect monotonic, then the local average treatment effect (LATE) of the treatment on compliers is identifiable and can be estimated using the Wald estimator: the ratio of the 
the instrument's effect on the outcome and on the treatment
\citep{angrist1996identification}. 
Monotonicity has been shown to be critical to identifiability in the IV setting in the presence of heterogeneous effects \citep{imbens1994identification}.
We could conceivably do the same after conditioning everywhere on covariates $X$ to obtain a conditional LATE (CLATE) \citep[see e.g.][]{aronow2013beyond,athey2019generalized}.
The ratio is then between the CATE of the instrument on the outcome and the CATE of the instrument on on the treatment, and the latter (but not the former) indeed assumes monotonic effect.
But for use in this conditional Wald estimator, we would actually be interested in estimating CATE itself rather than learning a classifier.
\citet{kennedy2018sharp}, however, use a classifier given by thresholding such a CATE estimate in order to focus on subgroups where compliance is high for better interpretability. The discriminative classifiers developed herein can instead be used in their method.

\section{A Discriminative Approach using Surrogate Losses}\label{sec:discriminative}

We next proceed to develop a new discriminative approach to classifying treatment responders from $X,T,Y$ data. The approach is based on leveraging monotonicity to reformulate the weighted misclassification loss $L_\theta(f)$ in terms of an expectation over observable quantities, interpreting this expectation as an average loss, and minimizing an empirical average of surrogate losses.

We begin by reformulating the weighted misclassification loss under causal effect monotonicity.
\begin{lemma}\label{lemma:lossreformulation}
Under \cref{asm:monotonicity},
\begin{align}
\notag
\textstyle L_\theta(f)=
&\frac14\underbrace{\Efb{f(X)(2\theta-YT/Q)}}_{L'_\theta(f)}
\\[-1em]\notag&\textstyle\qquad\qquad\qquad\qquad
+\frac14\Efb{2\theta+(1-2\theta)YT/Q}.
\end{align}
\end{lemma}
\begin{proof} We have that
\begin{align*}
&\textstyle\Prb{\text{false negative}}
\\&\textstyle=\Efb{\indic{f(X)=-1}\indic{R=+1}}
\\&\textstyle=\frac12\Efb{\prns{1-f(X)}\rho(X)}&&\text{(iter. expectations)}
\\&\textstyle=\frac14\Efb{\prns{1-f(X)}\Efb{YT/Q\mid X}}&&\text{(\cref{lemma:cate})}
\\&\textstyle=\frac14\Efb{\prns{1-f(X)}\prns{YT/Q}}&&\text{(iter. expectations)}.
\end{align*}
A symmetric argument similarly shows $\Prb{\text{false positive}}=\frac14\Efb{\prns{1+f(X)}\prns{2-YT/Q}}$. Combining the two, respectively weighted by $1-\theta$ and $\theta$, and collecting terms yields the above result.
\end{proof}
\Cref{lemma:lossreformulation} decomposes the misclassification loss $L_\theta(f)$ into two parts: a part that depends on $f$ ($L'_\theta(f)$) and a part that is independent of $f$. It therefore shows that optimizing $L_\theta(f)$ is the same as optimizing $L'_\theta(f)$.

Notice, moreover, that we can rewrite
\begin{equation}\label{eq:lossreformweighted}L'_\theta(f)=\Efb{W\ell(f(X),Z)},\end{equation}
where
$Z=\op{sign}(YT/Q-2\theta)$,
$W=\abs{YT/Q-2\theta}$, and
$$
\ell(\hat z,z)=
\pw{+1&\hat z\neq z\\-1&\hat z=z}.
$$
This shows that $L'_\theta(f)$ has the form of a weighted misclassification loss for the problem of trying to predict $Z$ from $X$, where each instance is weighted by $W$.

\subsection{A corrupted label perspective}\label{sec:corrupted}

We now give an interpretation of this reformulation from the perspective of a classification problem with corrupted label data. For the sake of exposition, suppose $Q=1/2$, that is, the data came from a Bernoulli trial with equal treatment probabilities. Then we have that
\begin{align*}
Z=YT,\quad
W=2(1-Z\theta).
\end{align*}
That is, $Z$ is a binary label indicating whether $Y=T$ or $Y\neq T$, and examples where $Y\neq T$ get $\frac{1+\theta}{1-\theta}\in[1,\infty]$ times the weight that examples where $Y=T$ get. For example, if $\theta=1/2$, then this weight ratio is $3$ to $1$. 

To understand this disparity, we will consider $Z$ as a surrogate label for the responder status $R$. 
To see that $Z$ can serve as a surrogate label for $R$ note that by \cref{lemma:lossreformulation}, $f_\theta^*$ minimizes $L'_\theta(f)$ and hence can also be understood as a classifier for $Z$.
Next, note that an example with responder status $R=+1$ will by definition have $Y=Y(T)=T$ and therefore $Z=+1$. On the other hand, an example with responder status $R=-1$ will either have $Z=-1$ if by chance the $T$ coin flip (recall $Q=1/2$) ends up opposite to the unit's non-responder type and otherwise $Z=+1$, so $Z$ will be $\pm1$ equiprobably. Therefore, $Z$ can be seen as a corrupted form of $R$, which aligns with $R$ whenever $R$ is positive and gets scrambled whenever $R$ is negative. As such, negative examples with $Z=-1$ are seen as more definitive and therefore carry more weight.

In \cref{sec:generative}, we also use this corrupted label perspective to develop a generative approach and a new cross-entropy loss.

\subsection{Weighted surrogate loss minimization}\label{sec:surrogatemin}

We now present our first proposal for a responder-classification algorithm.
The reformulation in \cref{eq:lossreformweighted} suggests the following discriminative responder-classification algorithm based on re-weighting surrogate-loss-based classification algorithms. Let $\mathcal H\subseteq[\R p\to\Rl]$ be a function class representing score functions, let $\ell'$ be a surrogate classification loss, and let $\Omega:\mathcal H\to\Rl_+$ be a potential regularizer. Then return the classifier 
\begin{equation}\label{eqn:surrogateclassifier1}
\hat f(x)=\op{sign}(\hat h(x))
\end{equation}
where $\hat h$ is the solution 
\begin{equation}\label{eqn:surrogateclassifier2}
\hat h\in\argmin_{h\in\mathcal H}
\frac1n\sum_{i=1}^nW_i\ell'(h(X_i),Z_i)+\Omega(h).
\end{equation}
For example, if $\ell'(\hat r,z)=\max(0,1-z\hat r)$, $\mathcal H$ is a reproducing kernel Hilbert space, and $\Omega(h)=\lambda\fmagd{h}^2$ is the squared norm in that space, then we get a sample-weighted support vector machine \citep{scholkopf2001learning}. We call the corresponding responder-classification algorithm \textbf{RespSVM}. As another example, if $\ell'(\hat r,z)=\log(1+e^{\hat r})-(1+z)\hat r/2$ and $\mathcal H$ is all neural networks of a given architecture then we get a sample-weighted neural network ($\Omega(h)$ may be nothing or it may be the sum of squared weights for weight decay). We call the corresponding responder-classification algorithm \textbf{RespNet-disc}, or \textbf{RespLR-disc} in case of a linear architecture with no hidden layers.

\subsection{What happens if monotonicity fails? A policy learning perspective}\label{sec:policylearning}

While it is self-evident when one is in a setting where outcome data are binary, monotonicity is always an assumption. Moreover since we do not see counterfactual outcomes, it may not have observable ramifications and may not be testable. This raises an important question: what happens if the monotonicity assumption fails? Can we still meaningfully interpret the classifier $\hat f$ that we learn in \cref{eqn:surrogateclassifier1}?

The next result shows how we can give an interpretation based on policy learning.
\begin{lemma}\label{lemma:policy}
Minimizing $L'_\theta(f)$ is the same as maximizing $U_\theta(f)=\Efb{Y(f(X))}-2\theta\Prb{f(X)=1}$.
\end{lemma}
\begin{proof}
This follows from the facts that $\Efb{Y(f(X))}=\Efb{\frac{1+f(X)}{2}Y(+1)-\frac{1-f(X)}{2}Y(-1)}=\frac12\Efb{Y(+1)+Y(-1)}+\Efb{f(X)YT/Q}$ (invoking \cref{eqn:CATE}) and $\Prb{f(X)=1}=\Efb{\frac{1+f(X)}{2}}=\frac12+\frac12\Efb{f(X)}$.
\end{proof}
We can interpret \cref{lemma:policy} as follows. Suppose outcomes are rewards, where positive outcomes are preferred to negative outcomes. Suppose the function $f$ is a policy mapping features to a decision to apply treatment ($+1$) or not ($-1$). And, suppose the cost of applying treatment is $2\theta$. Then $U_\theta(f)$ is the total average rewards minus costs incurred by following the policy $f$. Then, regardless of monotonicity holding or not, by minimizing $L'_\theta(f)$ (or an empirical surrogate version thereof) we are seeking a policy that achieves a good rewards-costs trade off. 

The policy learning perspective provides a useful frame even when monotonicity holds.
Notice that if monotonicity were true, then $\tau(X)\geq0$ and, in this reward interpretation of outcomes, every unit can only benefit from treatment. Correspondingly, if treatment had no cost ($\theta=0$) and monotonicity held, we would always set $f(X)=+1$. Indeed, this would minimize false negatives. However, if we were also concerned with false positives, we would not always predict positive. Indeed, even if everyone could only stand to benefit from treatment, if there was a cost to treatment and there was an uncertainty as to whether treatment would actually make a difference in a certain context, then perhaps the treatment should not be applied.

This perspective is closely related to the reduction by
\citet{beygelzimer2009offset,zhao2012estimating} of maximizing $\Efb{Y(f(X))}$ to cost-sensitive classification, which reduces to weighted misclassification in the binary treatment case. However, since effects are monotonic, it is clear that $f(x)=+1$ constant maximizes the above. To introduce the cost to treatment in this framework one would shift all treated outcomes by $\theta$. Doing this, however, produces a \emph{different} set of weights that depend not just on the label value of $Z$ but rather depend simultaneously on $Y$ and $T$.

\section{A Generative Approach}\label{sec:generative}

We next present a new generative approach to classifying treatment responders from $X,T,Y$ data.
In this approach, we will actually estimate the conditional probability $\rho$ directly using maximum likelihood.
The approach is closely related to the corrupted label perspective we presented in \cref{sec:corrupted}.
Without loss of generality, we can consider the data as being generated by first drawing $(X,R)$ and then corrupting the $R$ label to produce $(X,Z)$. We can then use maximum likelihood to fit a generative model to the probability of observing the label $Z=YT$, which we can phrase in terms of $\rho$.

For this section, we assume that treatment assignment is equiprobable so that $Q=1/2$. Alternatively, this can be achieve by weighting each unit by $1/Q$ to create a pseudo-population where this is the case.
Under this assumption we can relate
\begin{lemma}\label{lemma:probz}
Suppose $Q=1/2$ and \cref{asm:monotonicity} holds. Then
\begin{equation}\label{eq:probz}\textstyle
\Prb{Z=z\mid X}=\frac{1+z\rho(X)}2.
\end{equation}
\end{lemma}
\begin{proof}
Notice that since $R=1\implies Z=1$, we have\break
$\Prb{Z=z\mid X,R=1}=\frac{1+z}2$.
Let $A=\frac{Y(+1)+Y(-1)}2$, which is $-1$ for type-1 non-responders, $0$ for responders, and $+1$ type-2 non-responders.
By unconfoundedness and $Q=1/2$, we have $\Prb{T=t\mid X,A}=1/2$. Therefore,
\begin{align*}
&\Prb{Z=z\mid X,R=-1}\\&=
\Prb{A=-1\mid X,A\neq0}\Prb{T=-z\mid X,A=-1}\\&~\;+
\Prb{A=+1\mid X,A\neq0}\Prb{T=z\mid X,A=+1}=\textstyle\frac12.
\end{align*}
Hence, marginalizing over $R$,
$$\textstyle
\Prb{Z=z\mid X}=\rho(X)\frac{(1+z)}2+(1-\rho(X))\frac{1}2,
$$
which gives the result.
\end{proof}

\Cref{lemma:probz} shows how the corrupted label $Z$ is related to $R$ in their conditional probabilities. Note that, if $Q$ were not equal to $\frac12$ then $\Prb{A=+1\mid X,A\neq0}$ would not cancel out in the above and would remain as a nuisance parameter in the below estimation approach, that is, we would not be able to avoid also estimating the probabilities of being each type of non-responder. Having $Q=\frac12$ (\eg, by creating an appropriate weighted pseudo-population if it is not already the case) enables us to ignore this nuisance and focus just on $\rho(X)$.

\subsection{MLE of $\rho$ under monotonicity}

Based on \cref{lemma:probz}, we can formulate the negative log likelihood of observing the labels $Z_1,\dots,Z_n$ given the covariate design $X_1,\dots,X_n$ as a function of $\rho$ as a parameter:
\begin{equation}\label{eq:rhonll}\textstyle
\op{NLL}_n(\rho)=-\sum_{i=1}^n\log\prns{\frac{1+Z_i\rho(X_i)}2}.
\end{equation}
Then, given a class of probability models $\mathcal R\subseteq[\R p\to[0,1]]$ and potentially a regularizer $\Omega:\mathcal R\to\Rl_+$, we can estimate $\rho$ by optimizing a regularized maximum likelihood
\begin{equation}\label{eq:rhonllmin}\textstyle
\hat\rho\in\argmin_{\rho\in\mathcal R}\op{NLL}_n(\rho)+\Omega(\rho).
\end{equation}
Specifically, we focus on using this for neural networks, where $\mathcal R$ is neural networks of a given architecture (with a sigmoid activation at its output). We call the corresponding responder-classification algorithm \textbf{RespNet-gen}, or \textbf{RespLR-gen} in case of no hidden layers

\subsection{Comparison to weighted cross-entropy loss}

\begin{figure}[t!]\centering%
\begin{tabular}{lr}
\multicolumn{2}{c}{\includegraphics[width=0.35\textwidth]{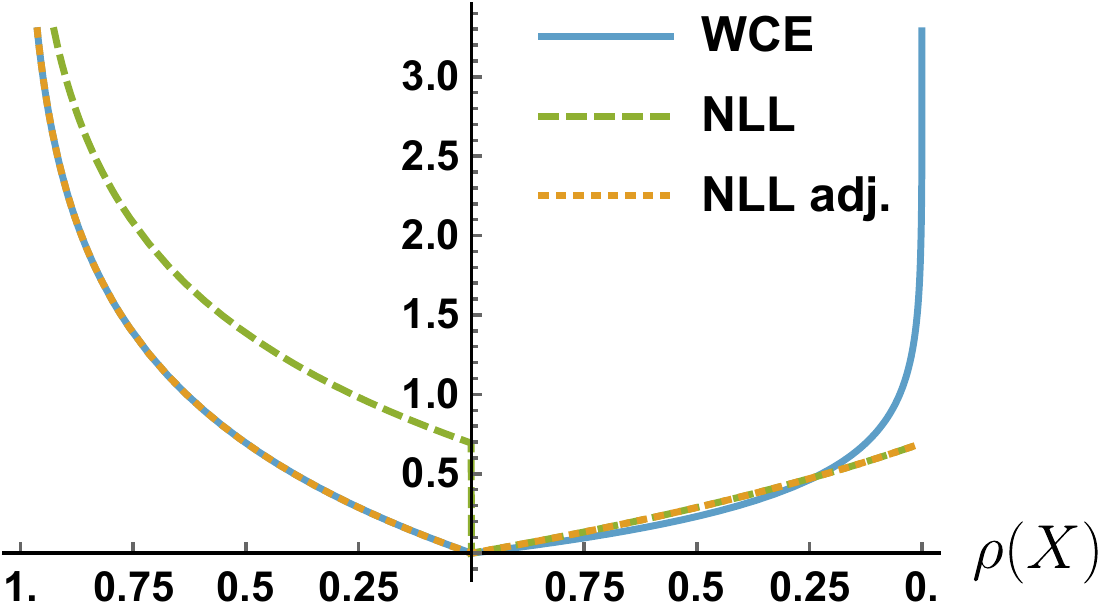}}\\%
\quad~~~$Z=-1$&$Z=+1$~~~\quad~~~~~~~~~~~~\\%
\end{tabular}%
\caption{Comparison of the weighted cross-entropy loss (WCE; \cref{eq:wce}), the negative log likelihood (NLL; negative log of \cref{eq:probz}), and the latter after subtracting $\log(2)$ from the $Z=-1$ branch. Note the tick marks on the horizontal axis.}\label{fig:nllvswce}%
\vspace{-0.75\baselineskip}\end{figure}

\begin{figure*}[t!]\centering%
\subfigure[The true label $R$]{\includegraphics[width=0.23\textwidth]{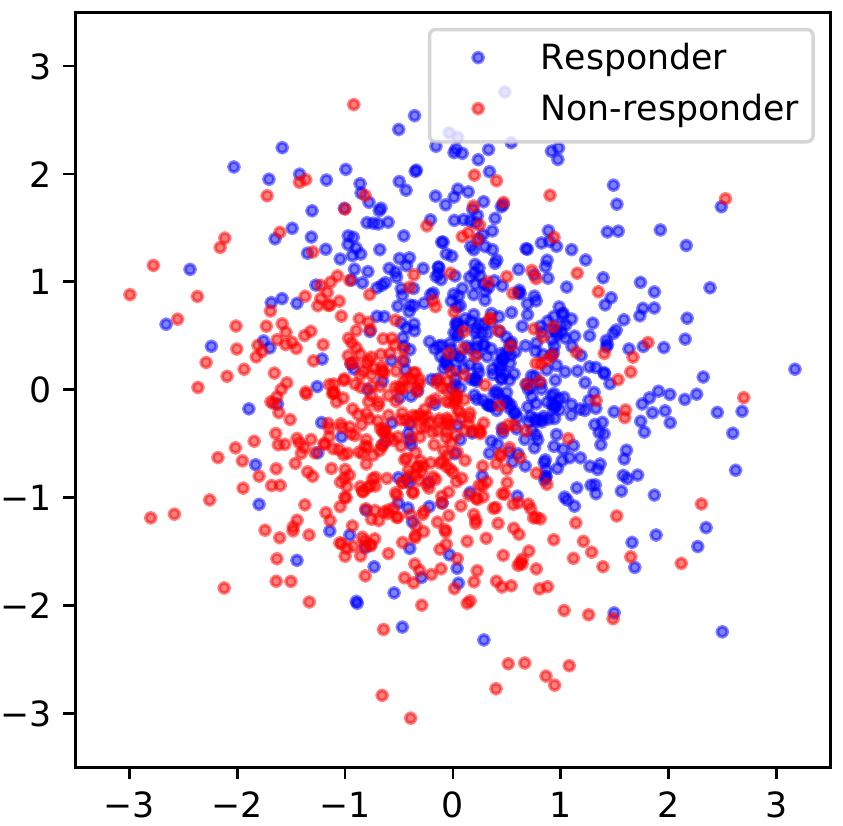}}%
\hspace{0.01\textwidth}%
\subfigure[The observable label $Z$]{\includegraphics[width=0.23\textwidth]{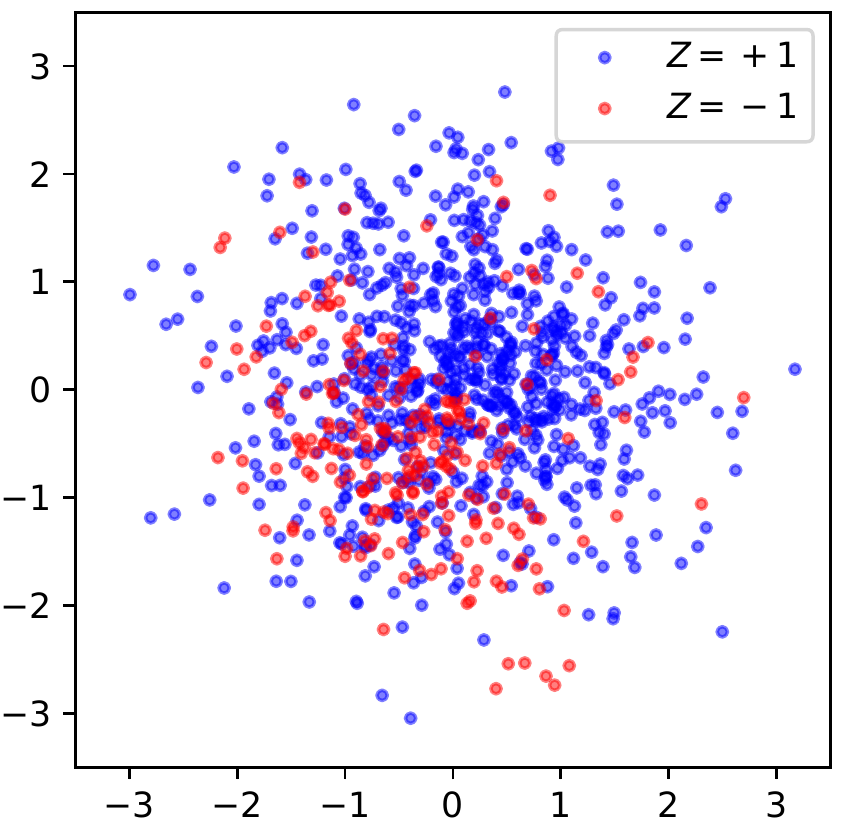}}%
\hspace{0.01\textwidth}%
\subfigure[$T=+1$]{\includegraphics[width=0.23\textwidth]{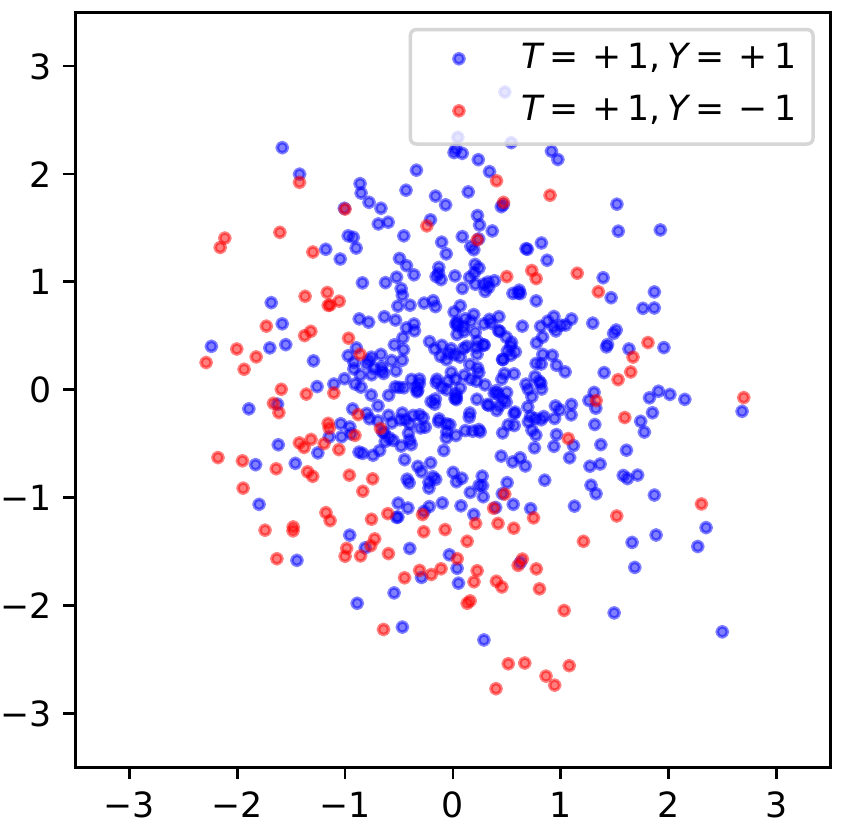}}%
\hspace{0.01\textwidth}%
\subfigure[$T=-1$]{\includegraphics[width=0.23\textwidth]{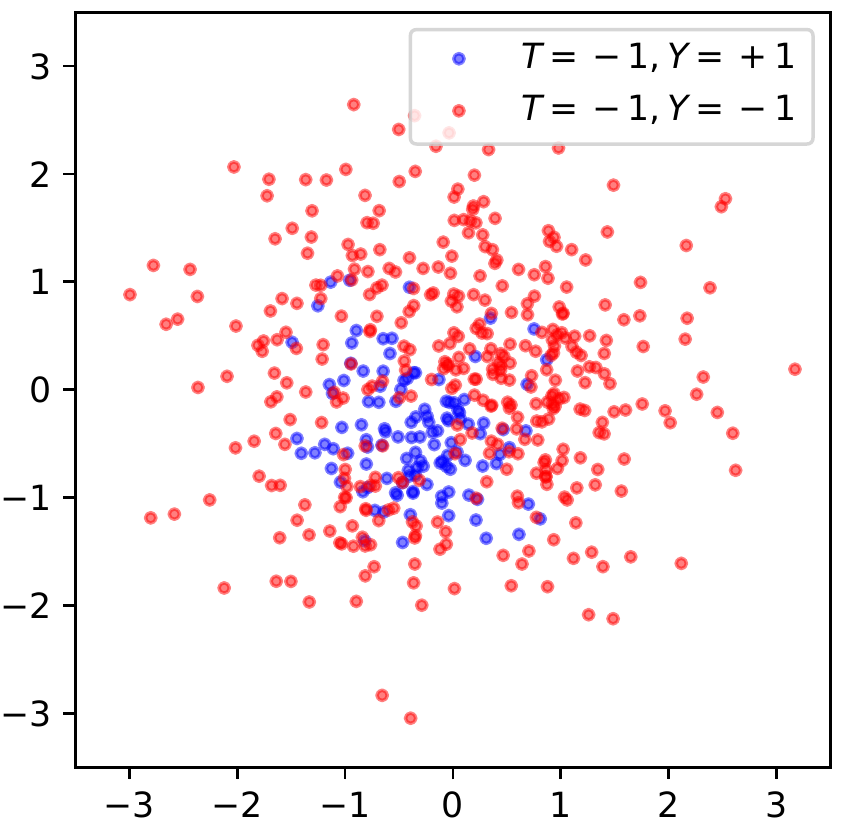}}%
\vspace{-0.5\baselineskip}
\caption{Linear scenario}\label{fig:lin}\end{figure*}

In \cref{sec:surrogatemin}, one proposal was to reweight and minimize the cross entropy loss. While the cross entropy loss serves both as a surrogate for misclassification \emph{and} as the negative likelihood objective for classification (when probabilities are set to the expit of the score), it is not the case for our problem even after reweighing. If $Q=\frac12$, the weighted cross entropy (WCE) loss applied to (the logit of) $\rho$ at a particular observation $(X,Z)$ as proposed in \cref{sec:surrogatemin} would be (after scaling by $1/3$)
\begin{equation}\label{eq:wce}\textstyle
-1\times\frac{1-Z}2\log(1-\rho(X))-\frac13\times\frac{1+Z}2\log(\rho(X)).
\end{equation}
This is distinct from the negative log of the likelihood, \cref{eq:probz}. The difference between these is shown in \cref{fig:nllvswce}, where we have stitched together the two cases $Z=\pm1$. The most noticeable feature is that the WCE touches 0 in both cases whereas NLL does not reach 0 when $Z=-1$. Indeed, even if the label $R$ is perfectly predictable from $X$, the label $Z$ is not. And, when $Z=-1$ we know that $R=-1$ necessarily, in which case observing $Z=\pm1$ was actually equiprobable and therefore the probability of $Z=-1$ can never be 1 (in fact, it is bounded by $1/2$) and hence NLL, its negative log, does not touch 0. But, fixing the label data, this amounts to constant shift in the loss function relative to $\rho$. Once we remove this shift (\ie, subtract $\sum_{i=1}^n\frac{1-Z_i}2\log2$ from $\op{NLL}_n(\rho)$ in \cref{eq:rhonll}), we obtain the curve denoted by NLL adj. in \cref{fig:nllvswce}. This matches WCE exactly in the $Z=-1$ case but is much flatter in the $Z=+1$ case and does not approach infinity as $\rho(X)\to0$. This permits the misclassification of $Z=+1$ labels, which indeed could have arisen from either $R=\pm1$ and therefore should not necessarily rule out $\rho(X)=0$. 

Note, nonetheless, that taking the conditional expectation of \cref{eq:wce} given $X$, differentiating by $\rho(X)$, and solving, gives back \cref{eq:probz} again. Same holds for the NLL loss. This shows that both approaches would be Fisher consistent. In practice, as explored in \cref{sec:empirics}, we find that the generative approach (RespNet-Gen) and its bounded loss in the $Z=+1$ case outperforms WCE (RespNet-Disc), although other surrogate losses such as hinge perform well.

\section{Empirical Studies}\label{sec:empirics}

\subsection{Synthetic datasets}

We first explore responder-classification on two synthetic datasets where we can more clearly illustrate and explain the behavior of different algorithms. We consider two scenarios for various covariate dimensions $d$. In \emph{both} scenarios we let $X\sim\mathcal N(0,I_d)$ be drawn a standard $d$-dimensional normal and $T=\pm1$ be drawn by an even coin flip. For each scenario we define $\rho(X)$ and $\alpha(X)$ below. To generate a data point given $X$ and $T$, we draw $R=\pm1$ as Bernoulli per $\rho(X)$ and $A=\pm1$ as Bernoulli per $\alpha(X)$. We then let $Y=\indic{R=+1}T + \indic{R=-1}A$.

We describe the two scenarios below:
\begin{enumerate}[leftmargin=*,,labelindent=0in,topsep=-1ex,itemsep=0ex,partopsep=0ex,parsep=0ex]
\item \textbf{Linear scenario}: $\rho(X)=0.15+0.7\indic{X_1>0}$, where $X_1$ is the first coordinate of $X$; and $\alpha(X)=1-F_{\op{Beta}(4,4)}(F_{{\chi_d^2}}(\fmagd{X}_2^2))$, where $F_{\op{Beta}(4,4)},\,F_{{\chi_d^2}}$ are the CDFs of a Beta and Chi-squared random variables, respectively. The parameters are chosen so that there are equal numbers of responder and non-responders and of type-1 and type-2 non-responders.
\item \textbf{Spherical scenario}: $\rho(X)=F_{\op{Beta}(4,4)}(F_{{\chi_d^2}}(\fmagd{X}_2^2))$; and $\alpha(X)=0.15+0.7\indic{\bigotimes_{j=1}^{d/2}(X_{2j-1}+X_{2j})>0}$, where $\bigotimes$ denotes the exclusive or (XOR) operation.
\end{enumerate}
An example draw with $d=2$ and $n=1000$ for the linear 
scenario is plotted in \cref{fig:lin}.
Panel (a) 
shows the true responder label, to which we do not have access at training. Panel
(b) shows the $Z=YT$ label, which we can observe. These figures illustrate why in order to solve the responder classification problem in panels (a) we should up-weight the $Z=-1$ examples. Panels (c) and (d) 
show the data in the treated and untreated groups, respectively. These show why it can often be harder to fit $\Prb{Y\mid X,T=\pm1}$ separately and take difference. In particular, this approach requires we actually estimate this probability rather than just lead some classifier in each of the treated and untreated sample.
\begin{figure*}[t!]\centering%
\subfigure[$d=2$]{\includegraphics[height=8.75em]{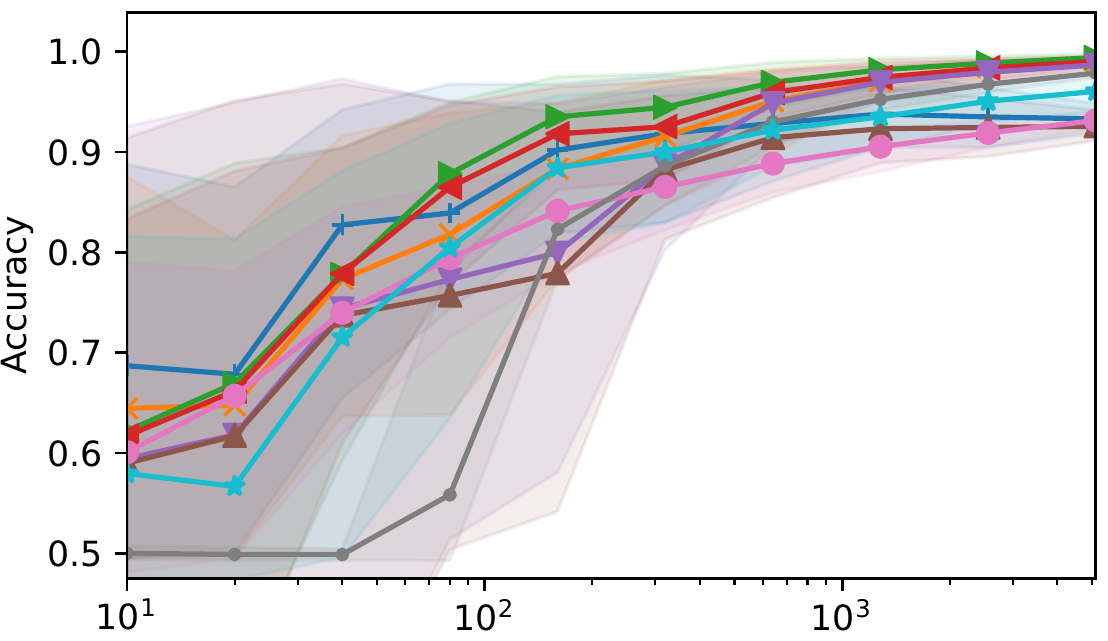}}%
\subfigure[$d=10$]{\includegraphics[height=8.75em]{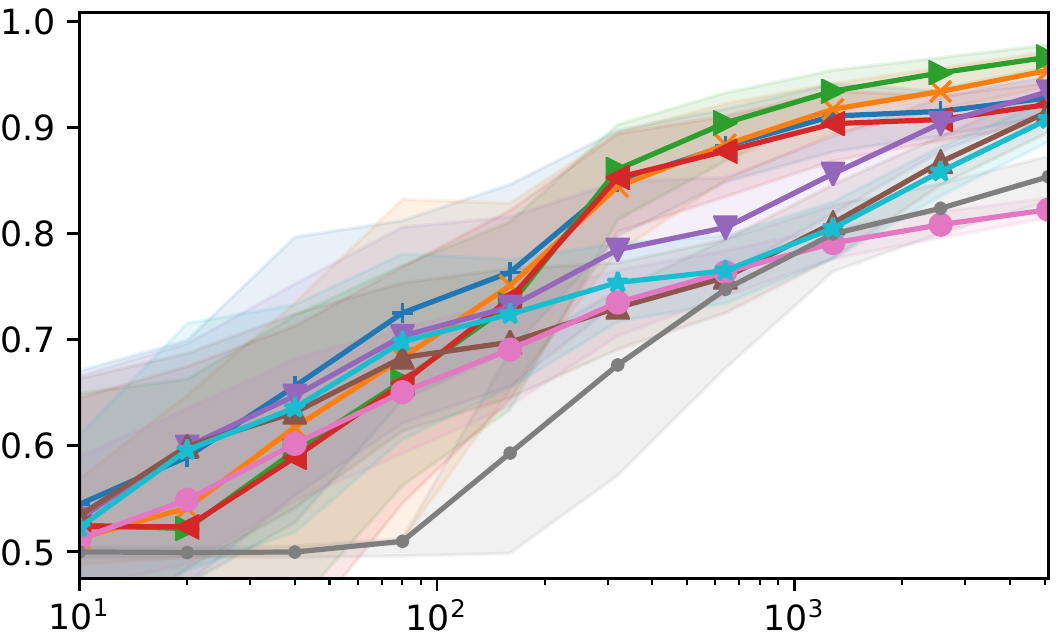}}%
\subfigure[$d=20$\qquad\qquad\quad~~~]{\includegraphics[height=8.75em]{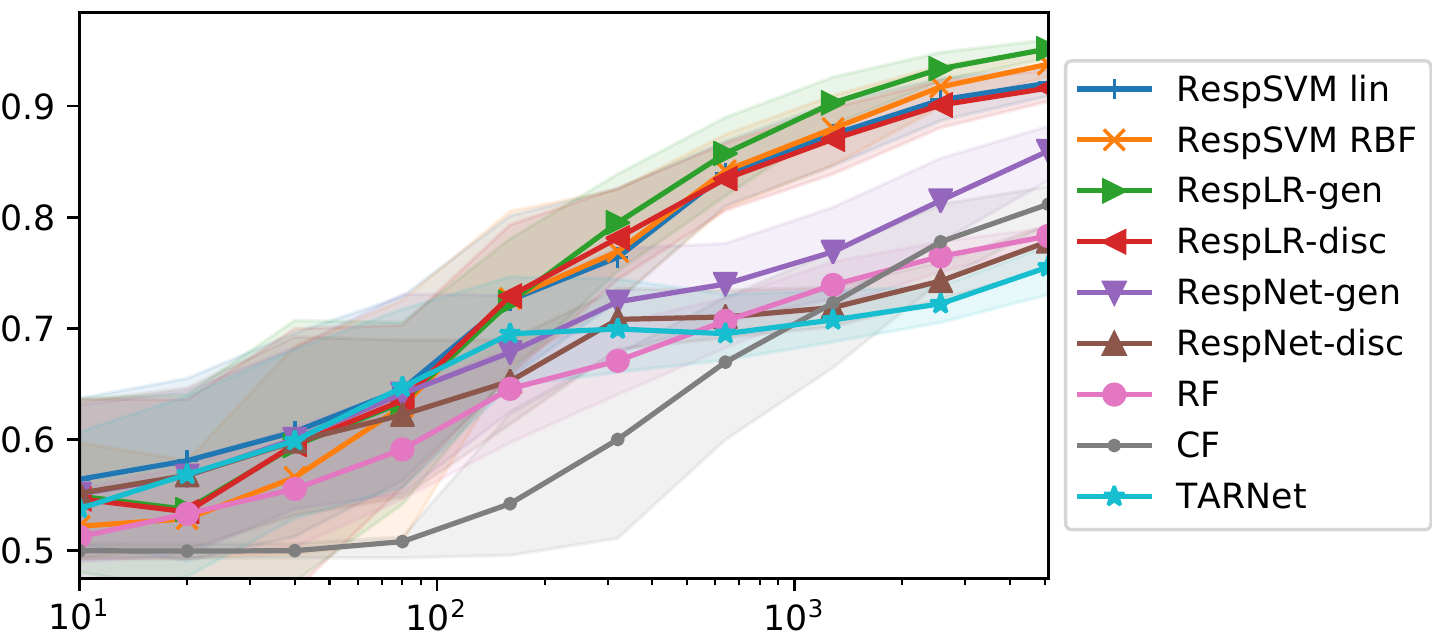}}%
\vspace{-1\baselineskip}
\caption{Accuracy results in the linear scenario as $n$ varies}\label{fig:linres}%
\end{figure*}
\begin{figure*}[t!]\centering%
\subfigure[$d=2$]{\includegraphics[height=8.75em]{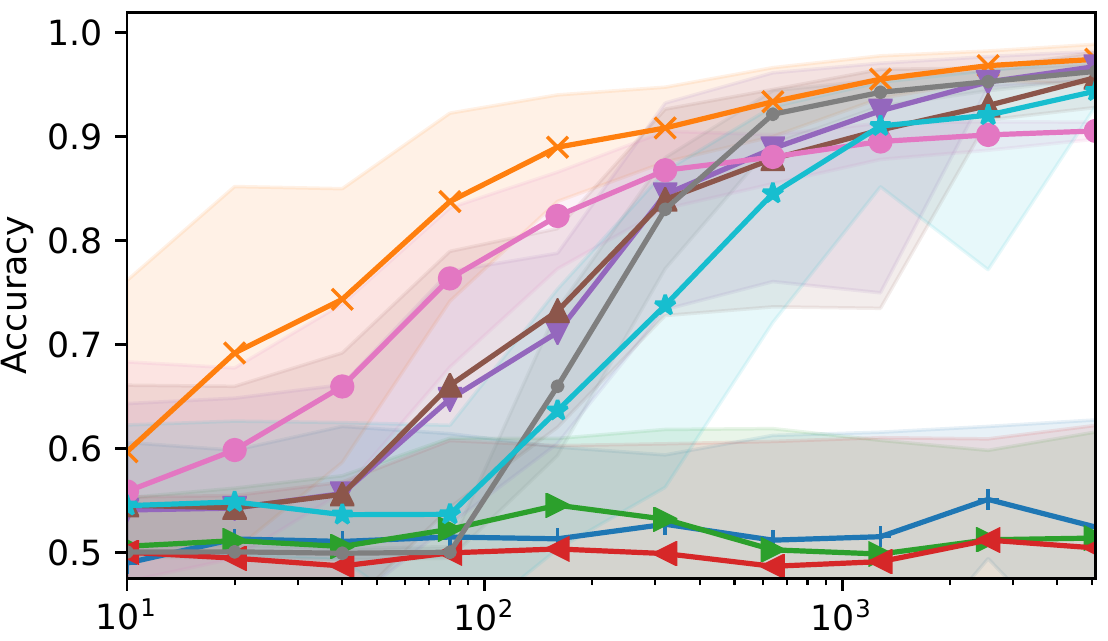}}%
\subfigure[$d=10$]{\includegraphics[height=8.75em]{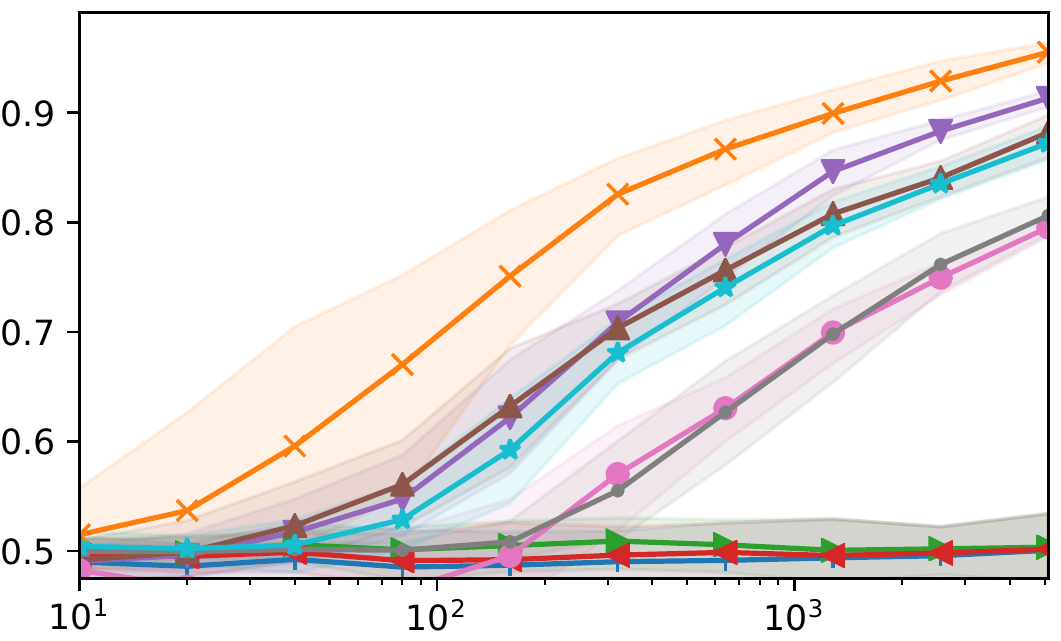}}%
\subfigure[$d=20$\qquad\qquad\quad~~~]{\includegraphics[height=8.75em]{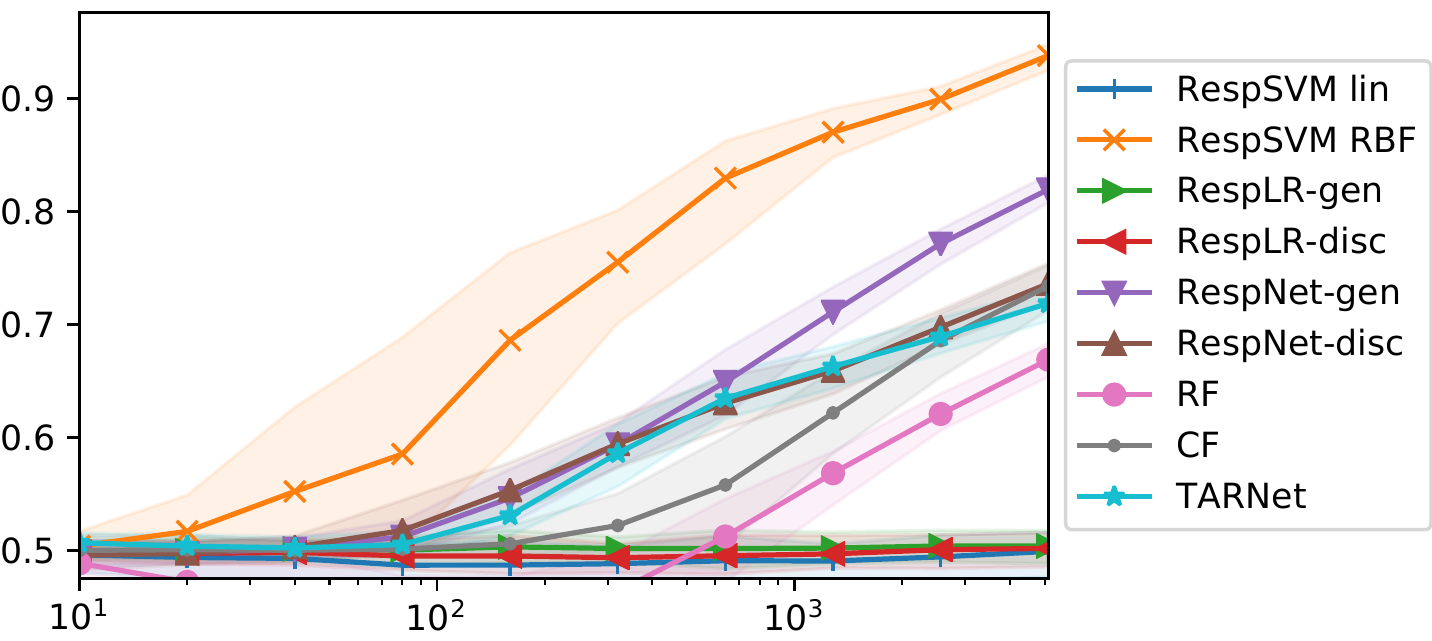}}%
\vspace{-1\baselineskip}
\caption{Accuracy results in the spherical scenario as $n$ varies}\label{fig:spherres}%
\end{figure*}

We next consider the performance of various approaches on these datasets. We focus on accuracy at predicting responder status, and hence set $\theta=1/2$. As benchmarks, we consider classifying responders by thresholding an estimate of CATE, as described in \cref{sec:cate}. We consider three prominent approaches to estimating CATE: using the differences of random forest regressions (RF; sklearn defaults as planned for v0.22, which increases default number of trees), using the causal forest method \citep[CF;][using \emph{R} package \texttt{grf} and defaults]{wager2017estimation}, and using a TARNet \citep{shalit2017estimating} with one shared hidden layer with $2d$ neurons and a hidden layer of $d$ neurons for each potential outcome with ELU activations in the interior and sigmoid activations at the outputs.
We compare these to the following variants of our methods: RespSVM with linear kernel and 5-fold cross validation (CV) to choose regularization (with $L'_\theta$ for scoring); RespSVM with RBF kernel and 5-fold CV to choose regularization and length-scale; and unregularized generative and discriminative RespNets with either no hidden layers (RespLR) or two hidden layers with $2d$ and $d$ neurons each and ELU interior activations (RespNet). RespNets and TARNets are implemented using Keras and TensorFlow and trained with Adam for 100 epochs.

In \cref{fig:linres,fig:spherres}, we plot average, 10$\thh$, and 90$\thh$ percentile accuracies of these methods in predicting $R^*=f_{\frac12}^*(X)$ as we vary $n$, $d$, and the scenario, each over 100 replications. In each scenario, we see a growing divergence between methods as we increase the dimension. In the linear case, the best methods overall are RespSVM (either kernel) and RespLR with the biggest improvements seen when $d$ is large and/or $n$ is small. In the spherical scenario, all linear models naturally fail and the best method overall is RespSVM with RBF kernel followed by RespNet-gen. Overall, the performance of RespNet-disc is similar to TARNet, which is improved upon by using the generative loss instead.

The results showcase that directly targeting the responder classification problem can be beneficial when it is of interest.
Note that the results do not imply that these existing algorithms are not good CATE learners; only that they can be improved upon in the special though common setting where outcomes are binary and effects are monotonic.

\subsection{Predicting response in decision to have third kid}\label{sec:childrearing}

We next study the application of our methods to the data derived from 1980 census. Following \citet{angrist1996children}, we construct a dataset of married couples with at least two children. We consider the treatment variable to be whether the biological sex of the two children at birth is the same and the outcome variable to be whether the couple has a third child or not. Thus, we are concerned with predicting whether the couple will respond or not to this treatment, and treatment is assigned at random equiprobably. (\citealp{angrist1996children} were originally interested in the effect of childbearing on women's participation in the labor force using the above as an instrument. Here we are only concerned in the choice of having a third kid as an outcome.)
As features we consider the ethnicity of the mother and of the father, their income and employment status, their ages at marriage, their ages at census, their ages at having their first kid and at having their second, their year of marriage, and the education level of the mother.

\begin{table}[t!]\centering
\vspace{-0.5\baselineskip}%
\caption{Results for child rearing example}\label{tab:childrearing}%
\begin{tabular}{lcccc}\toprule%
Method      & $L_\theta$ (in $0.01$) & \% 1st & \% 2nd & \% 3rd\\\midrule
RespSVM lin & $49\pm2.7$ & 100\%  &        &  \\
RespLR-gen  & $57\pm2.4$ &        & 100\%  &       \\
RespLR-disc & $58\pm2.3$ &        &        & 2\% \\
LR          & $58\pm2.3$ &        &        & 92\%\\
RF          & $58\pm2.3$ &        &        & 6\%
\\\bottomrule
\end{tabular}%
\vspace{-1\baselineskip}\end{table} 

To compare the different methods, we repeatedly draw two sets of $100,000$ units (observations of $X,T,Y$).
On one, we train each of our methods (after normalizing each column) and on the other we evaluate the false positives, false negatives, and weighted loss, using the result of \cref{lemma:lossreformulation} to estimate these. We set $\theta=0.0612$ to be equal to $\Efb{Z}$ so that always predicting positive or negative has the same weighted misclassification loss ($0.0574$). This allows us to focus on non-trivial improvements in balanced classification performance. We focus on the linear and forest-based methods as in the last section and add the difference of logistic regressions (LR, which is equivalent to TARNet with no hidden layers).
(We were not able to run CF on this size dataset, but this was likely due to limitations of the \texttt{rpy2} package.)
In \cref{tab:childrearing} we tabulate the average and standard deviation of $L_\theta$ over 50 replications and how often each of the methods produce the best, second best, or third best result. We find, as before, that the best performing methods are RespSVM and RespNet-gen (here with no hidden layers).

Finally, as an example of inference using these approaches, we consider the distribution of coefficients in the ResLR-gen model and construct 95\% Studentized bootstrap confidence intervals \citep{efron1994introduction}. We find that the only variables without a statistically significant influence on response at 0.05 significance are: father being white vs other, mother being black vs other, the age of the father at marriage, and the education of the mother being strictly more than high school vs no high school.

\section{Conclusions}

Predicting individual-level causal effects is an important problem. In this paper we specifically studied the arguably common setting where outcomes are binary and effect is monotonic, in which case this problem reduces to determining whether someone will respond to treatment. We formulated this as a classification problem, rather than a CATE estimation problem, and used this, together with monotonicity, to develop new methods for predicting individual-level causal effects. In their common but specialized setting they outperformed standard benchmarks.

\section*{Acknowledgements}

This material is based upon work supported by the National Science Foundation under Grant No. 1846210.

\bibliography{classresp}
\bibliographystyle{icml2019}

\end{document}